\documentclass[review]{elsarticle}

\usepackage[margin=2.5cm]{geometry}
\usepackage{amsmath,amssymb,amsthm}
\usepackage{booktabs}
\newtheorem{proposition}{Proposition}
\usepackage{graphicx}
\usepackage{xcolor}
\usepackage{algorithm}
\usepackage{algpseudocode}
\usepackage{subcaption}
\usepackage{pifont}
\usepackage{hyperref}
\usepackage{cleveref}
\usepackage{tikz}
\usetikzlibrary{shapes.geometric, shapes.misc, arrows.meta, positioning, calc}

\journal{Computers and Geotechnics}

\begin{document}

\begin{frontmatter}

\title{Formal verification of tree-based machine learning models for lateral spreading}

\author{Krishna Kumar\corref{cor1}}
\ead{krishnak@utexas.edu}
\cortext[cor1]{Corresponding author}
\affiliation[ut]{organization={Oden Institute of Computational Engineering and Sciences, The University of Texas at Austin}, city={Austin}, state={Texas}, country={USA}}

\begin{abstract}
Machine learning models for geotechnical hazard prediction can achieve high accuracy while learning physically inconsistent relationships from sparse or biased training data.
Current remedies (post-hoc explainability such as SHAP and LIME, and training-time constraints) either diagnose individual predictions approximately or restrict model capacity without providing exhaustive guarantees.
This paper encodes trained tree ensembles as logical formulas in a Satisfiability Modulo Theories (SMT) solver and checks physical specifications across the entire input domain, not just sampled points.
Four geotechnical specifications (water table depth, PGA monotonicity, distance safety, and flat-ground safety) are formalized as decidable logical formulas and verified via SMT against both XGBoost ensembles and Explainable Boosting Machines (EBMs) trained on the 2011 Christchurch earthquake lateral spreading dataset (7{,}291 sites, four features).
The SMT solver either produces a concrete counterexample where a specification fails or proves that no violation exists.
Unconstrained XGBoost (82.5\% accuracy) violates all four specifications.
Monotone-constrained XGBoost (69.3\% accuracy) satisfies PGA monotonicity but still violates compound specifications that no single-feature constraint can express.
The unconstrained EBM (80.1\% accuracy) violates all four specifications.
Imposing monotone constraints on PGA alone fixes Specification~B but leaves the remaining three violated.
A fully constrained EBM with monotone constraints on all four features (67.2\% accuracy) satisfies three of four specifications (water table depth safety remains violated by a thin margin), demonstrating that iterative constraint application guided by verification can progressively improve physical consistency.
A Pareto analysis of 33 model variants reveals a persistent trade-off, as among the variants studied, no model achieves both $>$80\% accuracy and full compliance with the specification set.
SHAP analysis of specification counterexamples shows that the offending feature can rank last among four features, demonstrating that post-hoc explanations do not substitute for formal verification.
These results establish a verify-fix-verify engineering loop for deploying physically consistent ML models in safety-critical geotechnical applications, and provide a foundation for formal certification of geotechnical ML models before field deployment.
\end{abstract}


\end{frontmatter}

%% ============================================================
\section{Introduction}
\label{sec:introduction}
%% ============================================================

Machine learning (ML) models increasingly predict earthquake-induced geotechnical hazards, including liquefaction triggering \citep{rateria-maurer-2022,geyin-2022,demir-sahin-2022} and lateral spreading \citep{hsiao-kumar-2024,durante-rathje-2021}.
These models capture complex, non-linear relationships among input features (peak ground acceleration (PGA), groundwater depth, distance to a free face, ground slope) that traditional empirical methods cannot represent \citep{zhang-2021}.
Ensemble methods like random forests \citep{breiman-2001} and gradient-boosted trees \citep{chen-guestrin-2016} achieve this by aggregating hundreds of weak predictors, each fitting a local partition of the feature space.
The aggregate captures interaction effects that no single decision tree could express, but it also obscures how individual features drive predictions.

This opacity creates a practical problem.
Models trained on sparse or biased geotechnical data can learn physically inconsistent relationships that remain invisible until deployment.
\citet{hsiao-kumar-2024} demonstrated this concretely.
Their XGBoost model, trained on the 2011 Christchurch earthquake lateral spreading dataset \citep{durante-rathje-2021}, predicted that the probability of lateral spreading \emph{decreased} as PGA increased in certain input regions, inverting the fundamental expectation that stronger shaking produces more ground displacement.
The model still achieved approximately 84\% overall accuracy because the non-physical pattern was confined to data-sparse input regions, precisely the regions where reliable predictions matter most.
This failure mode is not unique to one model or dataset.
\citet{maurer-sanger-2023} reviewed 75 AI liquefaction publications and identified systematic deficiencies (lack of comparison to state-of-practice models, departure from ML best practices, models discussed but not provided) that undermine trust in ML predictions across the field.
This paper addresses a complementary deficiency, the absence of exhaustive model-assurance tools that can determine whether a trained model satisfies specific consistency conditions, regardless of how it was trained or benchmarked.

Post-hoc explainability methods offer one line of defense.
SHapley Additive exPlanations (SHAP) \citep{lundberg-lee-2017} and Local Interpretable Model-agnostic Explanations (LIME) \citep{ribeiro-2016} have proven valuable for geotechnical ML \citep{can-2021,hsiao-kumar-2024}, but both are approximate.
\citet{huang-marques-silva-2024} proved that SHAP can assign non-zero importance to irrelevant features and zero importance to relevant ones.
Post-hoc methods can \emph{suggest} the presence of failure, but they cannot \emph{certify} its absence.

Correcting non-physical relationships requires modifying the model.
Monotone constraints \citep{chen-guestrin-2016} and shape function editing \citep{hsiao-kumar-rathje-2025} improve physical consistency but operate on individual features, cannot express compound conditions (e.g., ``far from free face \emph{and} weak shaking should not predict high risk''), and provide no proof that the corrected model satisfies the intended property across its full input domain.

This paper closes that gap by introducing formal verification via Satisfiability Modulo Theories (SMT) solvers to geotechnical ML.
We encode trained XGBoost and Explainable Boosting Machine (EBM) tree ensembles as logical formulas in an SMT solver and check whether physical specifications hold across all possible inputs simultaneously.
Tree ensembles partition the input space into finitely many hyperrectangular regions, each assigned a constant output.
Negating a physical specification and inlining this partition yields a quantifier-free formula in linear real arithmetic that the solver decides exactly, returning either \texttt{sat} with a concrete counterexample or \texttt{unsat} (specification proven).
We use Z3 \citep{z3-2008}, though the approach applies to any QF\_LRA-capable solver.
The contribution is methodological, a model-assurance workflow demonstrating how formal verification complements training-time constraints, rather than proposing a new predictor.
As geotechnical ML models move toward operational deployment, formal certification that these models respect established specifications becomes essential for safe adoption in practice.

The contributions of this paper are as follows.
\begin{enumerate}
    \item The first application of SMT-based formal verification to geotechnical ML models, demonstrating that both XGBoost ensembles and Explainable Boosting Machines trained on lateral spreading data can be encoded as logical formulas and formally checked against specified geotechnical specifications.
    \item Formalization of four geotechnical specifications, spanning single-feature thresholds (A), monotonicity (B), and compound multi-feature conditions (C, D).
    Each specification reduces to a quantifier-free satisfiability query in linear real arithmetic (QF\_LRA) after inlining the tree-ensemble score.
    The specific thresholds are calibrated to the Christchurch dataset and can be adapted to other sites.
    The verification framework itself is general and accepts any threshold values that local engineers deem appropriate.
    \item A verify-fix-verify engineering loop that uses SMT counterexamples to diagnose violations, applies targeted monotone constraints, and re-verifies to confirm or reveal residual failures.
    \item A Pareto analysis of 33 model variants quantifying the accuracy-consistency trade-off. Among the variants studied, no model achieves both $>$80\% accuracy and full compliance with the specification set, establishing a concrete cost for physical consistency that practitioners can evaluate.
\end{enumerate}

%% ============================================================
\section{Background and related work}
\label{sec:background}
%% ============================================================

%% ------------------------------------------------------------
\subsection{Machine learning for lateral spreading}
\label{sec:bg-ml}
%% ------------------------------------------------------------

Lateral spreading, the permanent horizontal ground displacement toward a free face during earthquake-induced liquefaction, damages buried utilities, bridge foundations, and building footings in seismic regions.
Predicting where lateral spreading will occur requires relating subsurface soil conditions and seismic loading to surface displacement, a task complicated by the heterogeneity of natural deposits and the rarity of strong-motion events.

\citet{durante-rathje-2021} compiled 7{,}291 sites from the 2011 Christchurch earthquake, each characterized by four features (GWD, distance to free-face $L$, slope, PGA) and labeled via remote sensing, training XGBoost classifiers achieving approximately 80--84\% accuracy.
\citet{demir-sahin-2022} confirmed that boosted tree ensembles consistently outperform logistic regression and SVMs on this dataset.
\citet{chen-ni-2024} extended the approach to displacement magnitude prediction, and \citet{hsiao-rathje-kumar-2025} incorporated CPT profiles via autoencoder, improving accuracy above 83\%.

\citet{hsiao-kumar-2024} found that the XGBoost model assigned \emph{negative} SHAP values to PGA in certain regions, meaning increasing shaking \emph{decreased} predicted spreading probability.
This non-physical behavior arose from sparse training data in the low-PGA, high-distance region.
SHAP revealed the problem visually but could not quantify its extent or guarantee that a corrected model would be free of similar artifacts.

\citet{hsiao-kumar-rathje-2025} replaced XGBoost with an EBM \citep{nori-2019,lou-caruana-gehrke-2012}, a Generalized Additive Model (GAM) decomposing the prediction into univariate shape functions $f_i(x_i)$ and pairwise interactions $f_{ij}(x_i, x_j)$.
\begin{equation}
g\bigl(\mathbb{E}[y]\bigr) = \beta_0 + \sum_{i} f_i(x_i) + \sum_{i<j} f_{ij}(x_i, x_j)
\label{eq:ebm}
\end{equation}
where $g$ is the logit link function.
Each $f_i$ is learned by cycling through features in round-robin fashion, accumulating per-feature updates into piecewise-constant functions.
This structure makes each $f_i$ inherently interpretable, unlike XGBoost where tree splits entangle multiple features.
\citet{hsiao-kumar-rathje-2025} found the PGA shape function exhibited a dip near 0.42~g and replaced it with a monotonic curve, sacrificing approximately 10~pp of accuracy.

These studies establish that tree ensembles achieve high accuracy on lateral spreading but can learn non-physical relationships from sparse data, and current remedies are manual and local.
Whether a corrected model satisfies physical specifications \emph{everywhere}, not just at sampled inputs, has not been addressed.

%% ------------------------------------------------------------
\subsection{Embedding physical knowledge in ML models}
\label{sec:bg-piml}
%% ------------------------------------------------------------

Several approaches embed physical knowledge into ML models, each targeting a specific architecture and constraint type.
Physics-informed neural networks \citep{raissi-2019} embed governing PDEs into the loss function, but are architecturally incompatible with tree-based classifiers and remain limited in geotechnical applications \citep{kumar-2025-pinns,yuan-2025}.
The constraints relevant to lateral spreading (``if groundwater is deeper than 5~m, do not predict spreading'') are logical conditions, not differential equations, requiring different enforcement mechanisms.

\paragraph{Hard monotone constraints}
Both XGBoost and EBM support \texttt{monotone\_constraints} parameters that force each specified feature's marginal effect to be non-decreasing or non-increasing during tree construction.
These constraints prevent single-feature violations but cannot express compound conditions involving multiple features simultaneously, because monotone constraints operate within individual tree splits, which test one feature at a time.

\paragraph{Post-hoc shape function editing}
\citet{hsiao-kumar-rathje-2025} inspected EBM shape functions, identified non-physical segments visually, and replaced them with monotonic curves.
The editing is effective but manual.
The analyst must decide which segments to modify and how to handle pairwise interactions whose non-physical behavior may not be visually obvious.
No guarantee exists that the edited model is globally consistent.

%% ------------------------------------------------------------
\subsection{Formal verification of ML models}
\label{sec:bg-fv}
%% ------------------------------------------------------------

Formal verification determines whether a system satisfies a specification through exhaustive mathematical reasoning.
Unlike testing, which evaluates a finite sample of inputs, formal verification reasons over the model's symbolic structure to produce either a proof that a property holds universally or a concrete counterexample where it fails.
Applied to ML models, the question becomes whether the learned input-output mapping satisfies a given property for \emph{all} possible inputs.

\paragraph{SMT solvers}
Satisfiability Modulo Theories (SMT) solvers \citep{barrett-seshia-2021} extend Boolean satisfiability (SAT) with decidable procedures for reasoning about arithmetic and other theories.
An SMT solver takes a formula and returns either \texttt{sat} with a satisfying assignment or \texttt{unsat}.
For instance, given $x_1 + x_2 > 10 \;\land\; x_1 < 0$, the solver deduces $x_2 > 10$ and returns a satisfying assignment such as $x_1 = -1,\; x_2 = 12$; tree-ensemble queries have exactly this structure, with each leaf contributing a constant to the sum.
Physical specifications for tree ensembles are stated as universally quantified first-order formulas.
Verification negates each specification and inlines the tree-ensemble score as a nested If-Then-Else (ITE) expression.
The resulting solver queries are quantifier-free formulas in linear real arithmetic with ITE terms (Quantifier-Free Linear Real Arithmetic, QF\_LRA), because every tree split compares a feature $x_j$ against a threshold $\theta$ and the ensemble output is a sum of constant leaf weights.
The decision procedure for QF\_LRA is complete, meaning the solver always terminates.
We use Z3 \citep{z3-2008} as our verification engine, though the methodology applies to any QF\_LRA-capable solver.

\paragraph{Tree ensemble verification}
\citet{kantchelian-2016} pioneered Mixed Integer Linear Programming (MILP)-based tree ensemble verification for computing optimal adversarial examples.
\citet{chen-zhang-2019} recast robustness verification as a max-clique problem, achieving speedups over MILP.
\citet{devos-2021} developed Veritas for generic verification tasks beyond adversarial robustness, and \citet{tornblom-2020} built VoTE for safety-critical autonomous systems.
\citet{sato-2020} proposed methods to extract violation ranges for constructing input filters.
On the explanation side, \citet{ignatiev-2019} developed formal abductive explanations (subset-minimal feature sets guaranteeing a prediction), extended by \citet{ignatiev-2022} to MaxSAT encodings.
\citet{varshney-2026} combined MILP and SMT into a data-aware sensitivity framework scaling to 800-tree ensembles.

All prior work focused on adversarial robustness or generic sensitivity.
No study has applied formal verification to check domain-specific physical specifications, which is the contribution of this paper.
The key insight is that tree ensembles partition the input space into finitely many hyperrectangular regions, each assigned a constant output.
Negating a physical specification and inlining this partition yields a decidable QF\_LRA formula.

%% ------------------------------------------------------------
\subsection{Positioning: formal verification as complement to constraints}
\label{sec:bg-position}
%% ------------------------------------------------------------

Hard constraints \emph{prevent} violations during training by restricting the hypothesis space, but provide no diagnostic information about where the unconstrained model would have failed.
Formal verification \emph{audits} a trained model by checking whether its learned function satisfies specified properties, returning either a proof or a concrete counterexample.
These two capabilities compose into a verify-fix-verify loop.
\begin{enumerate}
    \item Train an unconstrained model, optimizing purely for accuracy.
    \item Verify physical specifications formally. The solver provides concrete counterexamples for each violation.
    \item Analyze counterexamples to determine which constraints or corrections are needed.
    \item Apply targeted fixes (monotone constraints, shape function editing, or retraining with modified hyperparameters) guided by the counterexample analysis.
    \item Re-verify to confirm compliance or to discover residual violations that require additional iteration.
\end{enumerate}
This loop produces models whose physical consistency is not assumed but \emph{proven}.
The key distinction from simply applying constraints a priori is that verification reveals whether the constraints are sufficient.
As we demonstrate in \Cref{sec:axiom-results}, monotone constraints fix Specification~B (single-feature monotonicity) but leave Specifications~A, C, and D (threshold and compound conditions) unsatisfied, a failure that would remain undetected without formal verification.

%% ============================================================
\section{Problem formulation}
\label{sec:formulation}
%% ============================================================

%% ------------------------------------------------------------
\subsection{Dataset}
\label{sec:dataset}
%% ------------------------------------------------------------

We use the 2011 Christchurch earthquake lateral spreading dataset \citep{durante-rathje-2021}, comprising 7{,}291 sites (42\% spreading, 58\% not) classified by remote sensing displacement measurements exceeding 0.3~m.
The four features (\Cref{tab:features}) each have clear physical relationships to lateral spreading that domain experts can articulate as specifications.

\begin{table}[htbp]
\centering
\caption{Input features for lateral spreading classification.}
\label{tab:features}
\begin{tabular}{llll}
\toprule
Feature & Description & Range & Unit \\
\midrule
GWD & Groundwater depth & 0.37--6.05 & m \\
$L$ & Distance to nearest free-face & 0--3.29 & km \\
Slope & Ground surface slope & 0--10.92 & \% \\
PGA & Peak ground acceleration & 0.33--0.57 & g \\
\bottomrule
\end{tabular}
\end{table}

%% ------------------------------------------------------------
\subsection{Physical specifications as logical formulas}
\label{sec:axioms}
%% ------------------------------------------------------------

We formalize four physical specifications that any ML lateral spreading model should respect.
The four specifications are chosen to exercise structurally different formula classes (single-feature threshold, monotonicity, and compound multi-feature conditions) so that the verification framework is tested against the range of constraint types that arise in practice.
Our contribution is the verification framework, not the particular rules.

Let the input vector be
\begin{equation}
\mathbf{x} = (g, \ell, s, p) \in \mathcal{X} \subseteq \mathbb{R}^4,
\end{equation}
where $g$ is groundwater depth (GWD), $\ell$ is distance to free face ($L$), $s$ is ground slope, and $p$ is peak ground acceleration (PGA).
Let $f\colon \mathcal{X} \to \mathbb{R}$ denote the model's logit (log-odds) output, so that the predicted class is $\hat{y}(\mathbf{x}) = \mathbf{1}[f(\mathbf{x}) > 0]$.
\Cref{tab:notation} summarizes the notation and logical symbols used throughout.
Each specification is a universally quantified first-order formula over the bounded domain $\mathcal{X}$.
Verification reduces each to a quantifier-free satisfiability query after inlining the model as an ITE expression (\Cref{sec:encoding}).
The specifications encode \emph{necessary} conditions from geotechnical domain knowledge.

\begin{table}[t]
\centering
\caption{Notation and logical symbols.}
\label{tab:notation}
\begin{tabular}{cl}
\toprule
Symbol & Meaning \\
\midrule
$\mathbf{x} = (g, \ell, s, p)$ & Input vector (GWD, distance, slope, PGA) \\
$\mathcal{X} \subseteq \mathbb{R}^4$ & Bounded input domain \\
$f(\mathbf{x})$ & Model logit (log-odds) output \\
$\hat{y}(\mathbf{x})$ & Predicted class: $\mathbf{1}[f(\mathbf{x}) > 0]$ \\
\midrule
$\forall$ & Universal quantification (``for all'') \\
$\exists$ & Existential quantification (``there exists'') \\
$\neg$ & Logical negation (``not'') \\
$\Rightarrow$ & Logical implication (``if \ldots\ then'') \\
$\land$ & Logical conjunction (``and'') \\
$\lor$ & Logical disjunction (``or'') \\
$\bigwedge$ & Conjunction over a set (``and'' for all elements) \\
$\in$ & Set membership (``belongs to'') \\
$\leq,\; >$ & Arithmetic comparison \\
\midrule
\texttt{sat} & Satisfiable: counterexample exists (specification violated) \\
\texttt{unsat} & Unsatisfiable: no counterexample (specification proven) \\
\bottomrule
\end{tabular}
\end{table}
Satisfying them does not guarantee correctness in all respects.
The specific thresholds are calibrated to the Christchurch dataset and encode site-specific engineering bounds rather than universal physical laws.
Practitioners applying this framework elsewhere should recalibrate thresholds to local conditions.

\paragraph{Specification A (Water table depth)}
Sites with deep groundwater should not be predicted as lateral spreading.
\citet{geyin-2022} found that surface manifestation of liquefaction becomes negligible for GWD above 3--4~m, and we adopt a conservative 5~m threshold.
\begin{equation}
\forall\, (g,\ell,s,p) \in \mathcal{X}:\quad g > 5.0 \;\Rightarrow\; f(g,\ell,s,p) \leq 0
\label{eq:axiom-a}
\end{equation}
\paragraph{Specification B (PGA monotonicity)}
Increasing PGA should not decrease predicted risk, a first-principles consequence of soil dynamics that holds regardless of other features.
\begin{equation}
\forall\, (g,\ell,s,p_1),\, (g,\ell,s,p_2) \in \mathcal{X}:\quad p_1 \leq p_2 \;\Rightarrow\; f(g,\ell,s,p_1) \leq f(g,\ell,s,p_2)
\label{eq:axiom-b}
\end{equation}
Per-feature monotone constraints can directly enforce Specification~B by requiring $p$ to have a non-decreasing marginal effect during tree construction.
The same mechanism cannot enforce Specification~A, because monotone constraints control the \emph{direction} of a feature's marginal effect but not the absolute output value at a given threshold.
A model with a monotonically decreasing $g$ effect can still produce a positive logit at $g = 5.01$~m if the combined contributions from $\ell$, $s$, and $p$ are sufficiently large.

\paragraph{Specification C (Distance safety)}
Sites far from a free face experiencing weak shaking should not be predicted as lateral spreading.
The 2.5~km threshold exceeds the maximum documented lateral spread extent in Christchurch, and the 0.35~g threshold falls below confirmed spreading PGA values (minimum 0.37~g).
\begin{equation}
\forall\, (g,\ell,s,p) \in \mathcal{X}:\quad (\ell > 2.5 \;\land\; p < 0.35) \;\Rightarrow\; f(g,\ell,s,p) \leq 0
\label{eq:axiom-c}
\end{equation}
This specification constrains two features simultaneously, and no single-feature monotone constraint can express the conjunction.
Only 4 training observations fall within the premise region, none exhibiting spreading.

\paragraph{Specification D (Flat-ground safety)}
Specification~D restricts the same safety conclusion to a smaller, more physically certain subregion by adding the requirement that the ground surface is essentially flat ($s < 0.1$\%), eliminating the gently-sloping-ground driving mechanism.
\begin{equation}
\forall\, (g,\ell,s,p) \in \mathcal{X}:\quad (s < 0.1 \;\land\; \ell > 2.5 \;\land\; p < 0.35) \;\Rightarrow\; f(g,\ell,s,p) \leq 0
\label{eq:axiom-d}
\end{equation}
Only 1 observation in the training data satisfies the premise, and it is non-spreading.
Specification~D is not logically independent of Specification~C.
Because its premise implies that of Specification~C ($s < 0.1 \land \ell > 2.5 \land p < 0.35 \Rightarrow \ell > 2.5 \land p < 0.35$), any model satisfying Specification~C automatically satisfies Specification~D.
Specification~D is therefore a weaker requirement, included as a diagnostic.
When a model violates Specification~C, Specification~D determines whether the violation persists even in the smaller flat-ground subregion where the driving conditions are weakest.
\Cref{tab:axiom-taxonomy} summarizes the four specifications and their structural classes.
Standard training-time monotone constraints can directly enforce only Specification~B.

\begin{table}[htbp]
\centering
\caption{Specification taxonomy. Training-time monotone constraints can directly enforce only the monotonicity class (B). Threshold (A) and compound (C, D) specifications require formal verification to check.}
\label{tab:axiom-taxonomy}
\begin{tabular}{clll}
\toprule
Spec. & Name & Type & Monotone constraint sufficient? \\
\midrule
A & Water table depth  & Single-feature threshold & No \\
B & PGA monotonicity   & Single-feature monotonicity & Yes \\
C & Distance safety    & Compound multi-feature & No \\
D & Flat-ground safety & Compound multi-feature & No \\
\bottomrule
\end{tabular}
\end{table}

We verify all four specifications over $\mathcal{X} = [0.37,\,6.05] \times [0,\,3.29] \times [0,\,10.5] \times [0.33,\,0.57]$ ($g \times \ell \times s \times p$), which rounds the observed data ranges to clean bounds and excludes two outliers (one with Slope $= 10.9$\%, one with PGA $= 0.329$~g), covering 7{,}289 of 7{,}291 observations (99.97\%).
Verification guarantees apply only within $\mathcal{X}$; model behavior outside these bounds is extrapolation.

%% ------------------------------------------------------------
\subsection{SMT encoding of tree ensembles}
\label{sec:encoding}
%% ------------------------------------------------------------

A gradient-boosted tree ensemble with $T$ trees computes the logit as
\begin{equation}
f(\mathbf{x}) = \beta_0 + \sum_{t=1}^{T} h_t(\mathbf{x})
\label{eq:ensemble}
\end{equation}
where $\beta_0$ is a base score and each tree $h_t$ maps features to a leaf weight via binary splits.
We encode each tree recursively as an SMT If-Then-Else (ITE) expression:
\begin{equation}
\textsc{Encode}(n) =
\begin{cases}
w_n & \text{if } n \text{ is a leaf with weight } w_n \\[4pt]
\texttt{If}\bigl(x_{j_n} < \theta_n,\; \textsc{Encode}(n.\text{left}),\; \textsc{Encode}(n.\text{right})\bigr) & \text{if } n \text{ is internal}
\end{cases}
\label{eq:encode-recursive}
\end{equation}
The full ensemble logit is $f(\mathbf{x}) = \beta_0 + \sum_{t=1}^{T} \textsc{Encode}(\text{root}_t)$.
For example, a depth-2 tree with root split on PGA at 0.42~g (\Cref{fig:tree-encoding-a}) becomes
\begin{equation}
h_1(g,\ell,s,p) = \texttt{If}\!\bigl(p < 0.42,\;
  \texttt{If}(s < 0.05,\; w_1,\; w_2),\;
  \texttt{If}(p < 0.48,\; w_3,\; w_4)\bigr)
\label{eq:tree-example}
\end{equation}
Each leaf corresponds to a conjunction of linear inequalities defining a polyhedral region of the input space (e.g., leaf $w_1$ is reached when $p < 0.42 \land s < 0.05$).
Summing $T$ such ITE expressions yields the ensemble logit $f(g,\ell,s,p) = \beta_0 + \sum_t h_t(g,\ell,s,p)$.
For our 100-tree, depth-5 model, this produces a single symbolic expression over four variables with 2{,}380 leaf terms.

\Cref{fig:tree-encoding} illustrates the encoding. Each root-to-leaf path becomes a conjunction of linear inequalities (left), and the full ensemble partitions the input space into regions of constant output (right, over $g \times p$), each corresponding to a unique leaf combination across all 100 trees.
Because this partition is finite, the solver reasons over all regions in a single query.

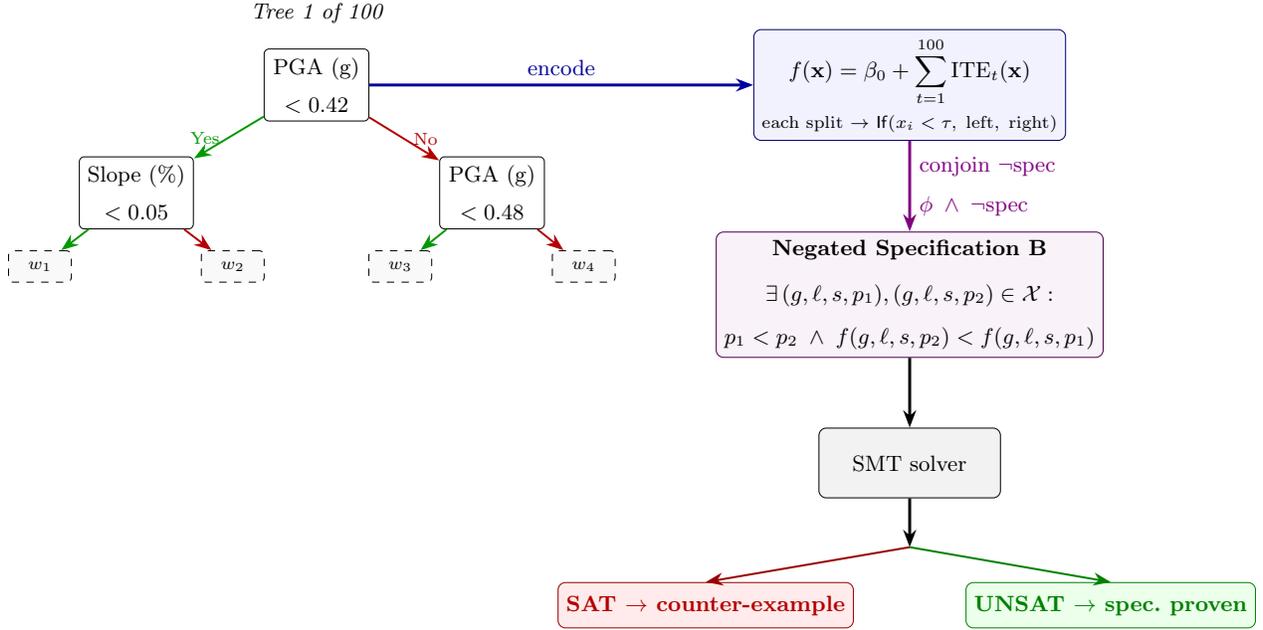
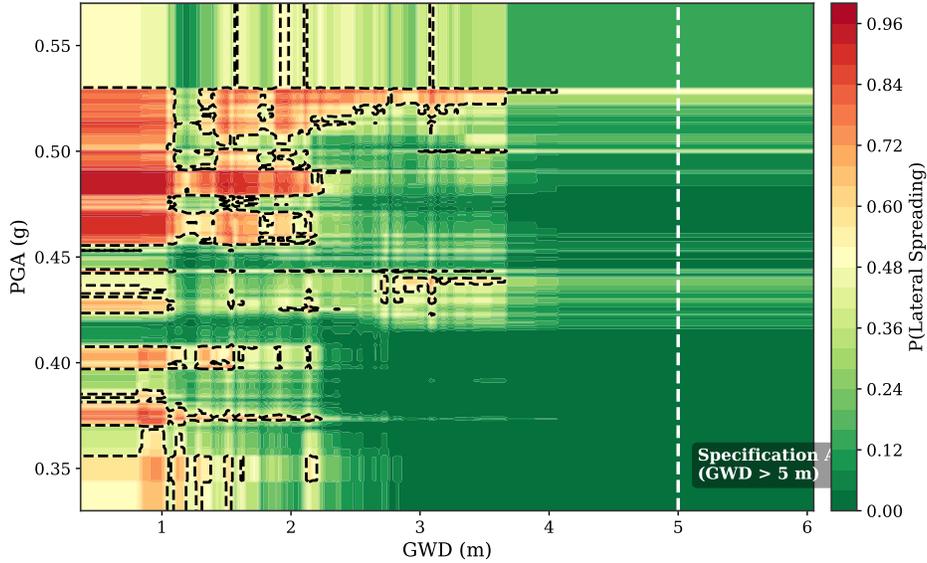
\begin{figure}[htbp]
\centering
\begin{subfigure}[b]{\textwidth}
    \centering
    \resizebox{\textwidth}{!}{%
    \begin{tikzpicture}[
        node distance=0.8cm and 1.0cm,
        split/.style={rectangle, draw, rounded corners=2pt, minimum width=1.5cm, minimum height=0.65cm, font=\small, align=center},
        leaf/.style={rectangle, draw, dashed, rounded corners=2pt, minimum width=0.9cm, minimum height=0.45cm, font=\scriptsize, align=center, fill=gray!5},
        proc/.style={rectangle, draw, rounded corners=3pt, minimum width=2.6cm, minimum height=1.0cm, font=\small, align=center},
        result/.style={rectangle, draw, rounded corners=3pt, minimum width=2.2cm, minimum height=0.65cm, font=\small\bfseries, align=center},
        arr/.style={-{Stealth[length=2.5mm]}, thick},
    ]
    % --- Tree (left) ---
    \node[split] (root) {PGA (g)\\[-1pt]$< 0.42$};
    \node[split, below left=0.5cm and 1.0cm of root] (n1) {Slope (\%)\\[-1pt]$< 0.05$};
    \node[split, below right=0.5cm and 1.0cm of root] (n2) {PGA (g)\\[-1pt]$< 0.48$};
    \node[leaf, below left=0.3cm and 0.1cm of n1] (l1) {$w_1$};
    \node[leaf, below right=0.3cm and 0.1cm of n1] (l2) {$w_2$};
    \node[leaf, below left=0.3cm and 0.1cm of n2] (l3) {$w_3$};
    \node[leaf, below right=0.3cm and 0.1cm of n2] (l4) {$w_4$};

    \draw[arr, green!60!black] (root) -- node[left, font=\scriptsize] {Yes} (n1);
    \draw[arr, red!70!black] (root) -- node[right, font=\scriptsize] {No} (n2);
    \draw[arr, green!60!black] (n1) -- (l1);
    \draw[arr, red!70!black] (n1) -- (l2);
    \draw[arr, green!60!black] (n2) -- (l3);
    \draw[arr, red!70!black] (n2) -- (l4);

    % Tree label
    \node[above=0.25cm of root, font=\small\itshape] {Tree 1 of 100};

    % --- SMT formula box ---
    \node[proc, right=5.5cm of root, fill=blue!5, draw=blue!40!black] (smt) {
        $f(\mathbf{x}) = \beta_0 + \displaystyle\sum_{t=1}^{100} \mathrm{ITE}_t(\mathbf{x})$\\[3pt]
        {\scriptsize each split $\to$ \textsf{If}$(x_i < \tau,\; \text{left},\; \text{right})$}
    };

    % --- Encode arrow (tree -> SMT box) ---
    \draw[arr, blue!60!black, line width=1.2pt] (root.east) -- (smt.west) node[midway, above, font=\small] {encode};

    % --- Negated axiom box ---
    \node[proc, below=1.3cm of smt, fill=violet!5, draw=violet!60!black] (axiom) {
        \textbf{Negated Specification B}\\[2pt]
        $\exists\,(g,\ell,s,p_1),(g,\ell,s,p_2) \in \mathcal{X}:$\\[1pt]
        $p_1 < p_2 \;\wedge\; f(g,\ell,s,p_2) < f(g,\ell,s,p_1)$
    };

    % --- Conjoin arrow (SMT formula -> Negated axiom) ---
    \draw[arr, violet, line width=1.2pt] (smt.south) -- (axiom.north) node[midway, right, font=\small, align=left] {conjoin $\neg$spec\\$\phi \;\wedge\; \neg$spec};

    % --- Z3 solver box ---
    \node[proc, below=1.0cm of axiom, fill=gray!10, draw=black] (z3) {SMT solver};

    % --- Arrow from negated axiom to Z3 ---
    \draw[arr, line width=1.2pt] (axiom.south) -- (z3.north);

    % --- Branch point below Z3 ---
    \coordinate[below=0.7cm of z3] (branch);
    \draw[arr, line width=1.2pt] (z3.south) -- (branch);

    % --- SAT / UNSAT outcomes ---
    \node[result, below left=0.5cm and 0.8cm of branch, fill=red!8, draw=red!60!black, text=red!70!black] (sat) {SAT $\to$ counter-example};
    \node[result, below right=0.5cm and 0.8cm of branch, fill=green!8, draw=green!50!black, text=green!50!black] (unsat) {UNSAT $\to$ spec.\ proven};

    \draw[arr, red!60!black] (branch) -- (sat.north);
    \draw[arr, green!50!black] (branch) -- (unsat.north);

    \end{tikzpicture}%
    }
    \caption{Verification pipeline. An XGBoost tree (left) is encoded as a nested SMT If-Then-Else expression. The negation of a physical specification is conjoined with the formula. If the solver returns SAT, it produces a concrete counter-example; if UNSAT, the specification is proven to hold over the \emph{entire} input domain.}
    \label{fig:tree-encoding-a}
\end{subfigure}

\vspace{1em}

\begin{subfigure}[b]{0.75\textwidth}
    \centering
    \includegraphics[width=\textwidth]{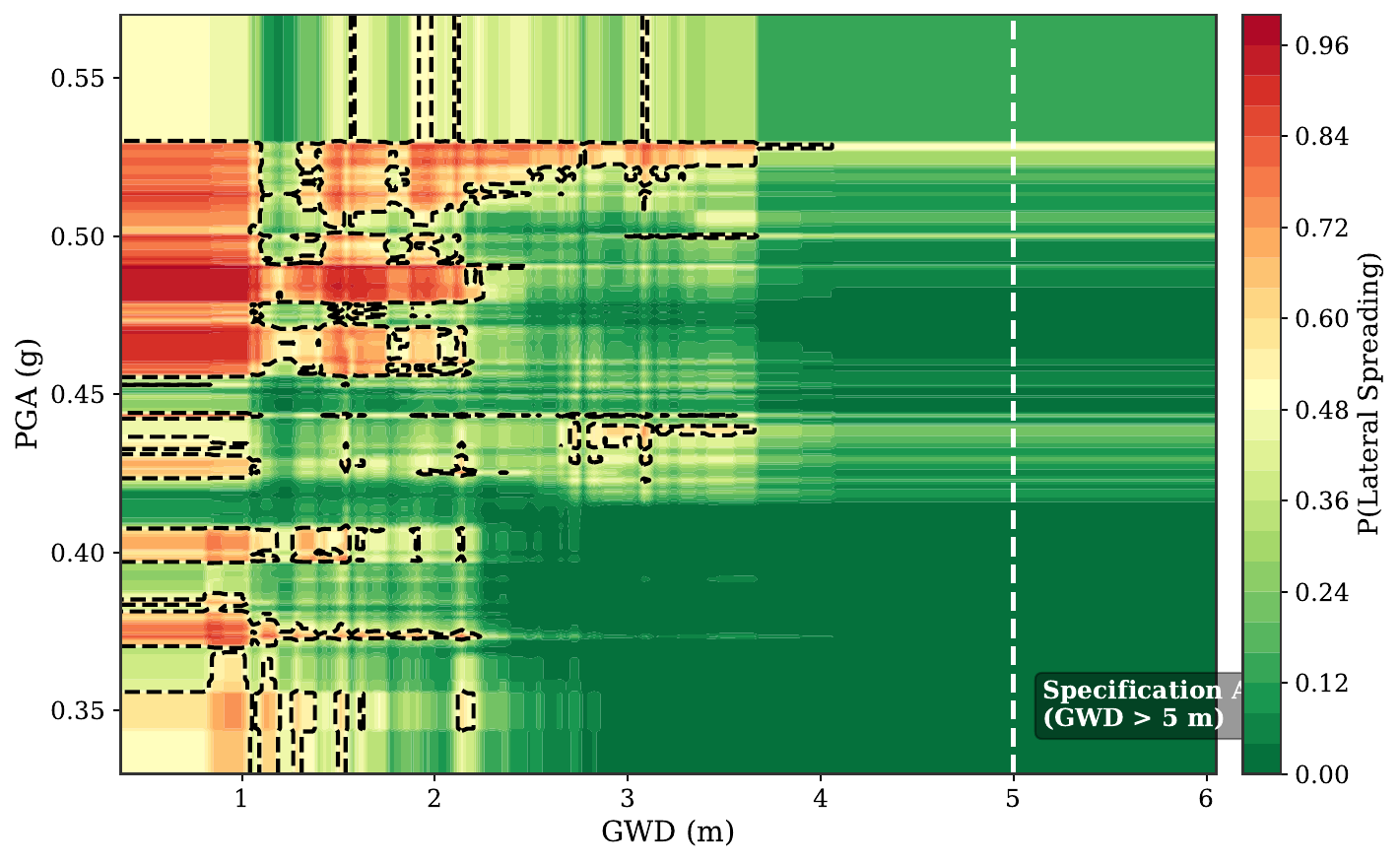}
    \caption{The resulting input space partition over $g$ and $p$ (with $\ell = 0.5$~km, $s = 2$\%), colored by predicted lateral spreading probability. Each colored region corresponds to a unique leaf combination across all 100 trees. The dashed line marks the Specification~A threshold ($g > 5$~m). The SMT solver reasons over all regions simultaneously.}
    \label{fig:tree-encoding-b}
\end{subfigure}
\caption{SMT encoding of an XGBoost gradient-boosted ensemble. Each of the 100 trees~(a) is encoded as a nested If-Then-Else expression over real-valued leaf weights; the ensemble sum of all trees partitions the input space into regions of constant predicted probability~(b).}
\label{fig:tree-encoding}
\end{figure}

Two implementation details are critical.
First, the base score $\beta_0$ is stored separately from the tree dump and must be included.
Omitting it shifts every logit by a constant offset.
Second, XGBoost uses 32-bit floats while SMT solvers use exact arithmetic, so we cast all thresholds and input features to \texttt{float32} to prevent precision mismatches.
Every counterexample was validated against the actual XGBoost model to confirm exact logit agreement.

To verify a specification, we assert its \emph{negation}.
Each specification has the form $\forall\,\mathbf{x} \in \mathcal{X}\!:\, P(\mathbf{x})$.
By the logical equivalence $\neg(\forall\,\mathbf{x}\!:\, P(\mathbf{x})) \;\equiv\; \exists\,\mathbf{x}\!:\, \neg P(\mathbf{x})$, ``it is not the case that $P$ holds for all $\mathbf{x}$'' is the same as ``there exists some $\mathbf{x}$ where $P$ fails.''
A specification is therefore violated if and only if such a counterexample exists.
The solver does not enumerate inputs.
Instead, for Specification~A (``for all inputs with GWD $> 5$~m, the model should not predict spreading''), the solver asks whether there exists an input where GWD $> 5$~m \emph{and} the model predicts spreading.
For Specification~C, the negated query is
\begin{equation}
\exists\, (g,\ell,s,p) \in \mathcal{X}: \quad \ell > 2.5 \;\land\; p < 0.35 \;\land\; f(g,\ell,s,p) > 0
\label{eq:verify-c}
\end{equation}
After inlining $f$ as an ITE expression, this becomes a quantifier-free QF\_LRA formula over four free real variables with domain constraints.
If the solver returns \texttt{sat}, the satisfying assignment (input) is a counterexample and the specification is violated.
If \texttt{unsat}, no such input exists and the specification holds universally.

%% ============================================================
\section{Methodology}
\label{sec:methodology}
%% ============================================================

%% ------------------------------------------------------------
\subsection{Tree-to-SMT encoding}
\label{sec:tree-encoding}
%% ------------------------------------------------------------

We parse XGBoost's JSON tree dump, which represents each tree as a nested dictionary of internal nodes (with split feature, threshold, and child pointers) and leaf nodes (with weight values).
\Cref{alg:encode} describes the recursive encoding.

\begin{algorithm}[htbp]
\caption{Recursive encoding of a decision tree as an SMT expression.}
\label{alg:encode}
\begin{algorithmic}[1]
\Function{EncodeTree}{node, features}
    \If{node is leaf}
        \State \Return $\texttt{RealVal}(\texttt{float32}(\text{node.weight}))$
    \EndIf
    \State $j \gets \text{node.split\_feature}$
    \State $\theta \gets \texttt{float32}(\text{node.threshold})$ \Comment{Match XGBoost arithmetic}
    \State $\text{left} \gets \Call{EncodeTree}{\text{node.yes\_child}, \text{features}}$
    \State $\text{right} \gets \Call{EncodeTree}{\text{node.no\_child}, \text{features}}$
    \State \Return $\texttt{If}(\text{features}[j] < \theta,\; \text{left},\; \text{right})$
\EndFunction
\end{algorithmic}
\end{algorithm}

For an ensemble of $T$ trees, the logit expression is
\begin{equation}
\texttt{logit} = \beta_0 + \sum_{t=1}^{T} \Call{EncodeTree}{\text{tree}_t.\text{root}, \text{features}}
\end{equation}

The resulting expression is a deeply nested ITE term.
A model with 100 trees of depth~5 produces up to $100 \times 2^5 = 3{,}200$ leaf terms (2{,}380 in our model), though the solver's internal simplification reduces this substantially before solving.

%% ------------------------------------------------------------
\subsection{EBM-to-SMT encoding}
\label{sec:ebm-encoding}
%% ------------------------------------------------------------

EBMs decompose the prediction into univariate shape functions and pairwise interactions, each stored as a lookup table over learned bin edges.
Each univariate function becomes a chain of ITE expressions over its bin edges:
\begin{equation}
f_i(x_i) = \texttt{If}(x_i < b_1,\; s_1,\; \texttt{If}(x_i < b_2,\; s_2,\; \ldots\; s_{n_i}))
\label{eq:ebm-ite}
\end{equation}
Adjacent bins with identical scores are merged, compressing 4{,}043 univariate bins to approximately 2{,}900 ITE nodes.
Pairwise interactions become nested two-dimensional ITE expressions (approximately 23{,}400 nodes for six interaction terms).
The full EBM logit is $f(\mathbf{x}) = \beta_0 + \sum_{i} f_i(x_i) + \sum_{i<j} f_{ij}(x_i, x_j)$.
We validated the encoding against EBM predictions on 20 test samples, achieving agreement to $< 10^{-12}$.

%% ------------------------------------------------------------
\subsection{How SMT verification works: a worked example}
\label{sec:smt-how}
%% ------------------------------------------------------------

We illustrate by checking Specification~A against a simplified two-tree ensemble.
Suppose tree~1 splits on $g$ at 2.5~m, then on $p$ at 0.42~g, and tree~2 splits on $\ell$ at 1.0~km:
\begin{align}
h_1 &= \texttt{If}(g < 2.5, \;\texttt{If}(p < 0.42, \; 0.3, \; 0.5), \;\texttt{If}(p < 0.42, \; -0.1, \; 0.05)) \\
h_2 &= \texttt{If}(\ell < 1.0, \; 0.2, \; -0.3)
\end{align}
The negated Specification~A query asks whether there exists $(g,\ell,s,p) \in \mathcal{X}$ such that $g > 5 \;\land\; (h_1 + h_2) > 0$.
The solver propagates $g > 5$ (and therefore $g \geq 2.5$) into tree~1's right branch, yielding either $-0.1$ or $0.05$.
If tree~1 outputs $0.05$ ($p \geq 0.42$) and tree~2 outputs $0.2$ ($\ell < 1.0$), the logit is $0.25 > 0$, so the solver returns \texttt{sat} with a counterexample (e.g., $g = 5.5$, $\ell = 0.5$, $p = 0.45$).

The solver reasons symbolically over the entire piecewise-constant structure, checking every region simultaneously via constraint propagation.
For our 100-tree model, each specification check typically completes in under 5~seconds.

%% ------------------------------------------------------------
\subsection{Verification procedure}
\label{sec:verify-proc}
%% ------------------------------------------------------------

For Specifications~A, C, and D, the solver checks a single query by asserting the specification's negation along with physical range constraints, and returns \texttt{sat} with a validated counterexample or \texttt{unsat} (specification proven).
Specification~B (monotonicity) is more expensive because it compares the model at two inputs differing only in PGA.
A direct encoding declares two feature vectors, constrains non-PGA features to be equal, and asserts that higher PGA produces lower logit, doubling the formula size and requiring 480~seconds for the monotone model.
We resolve this cost by leveraging the additive ensemble structure.

\begin{proposition}[Additive decomposition for monotonicity]\label{prop:decomp}
Let $f(\mathbf{x}) = \beta_0 + \sum_{m=1}^{M} g_m(\mathbf{x})$ be an additive model.
If each component $g_m$ is monotonically non-decreasing in feature $x_j$, then $f$ is monotonically non-decreasing in $x_j$.
\end{proposition}

\begin{proof}
Fix inputs $\mathbf{x}, \mathbf{x}'$ differing only in $x_j$ with $x_j < x_j'$.
By hypothesis, $g_m(\mathbf{x}) \leq g_m(\mathbf{x}')$ for each $m = 1, \ldots, M$.
Summing gives $\sum_m g_m(\mathbf{x}) \leq \sum_m g_m(\mathbf{x}')$; adding $\beta_0$ to both sides gives $f(\mathbf{x}) \leq f(\mathbf{x}')$.
\end{proof}

\Cref{prop:decomp} reduces ensemble monotonicity to $T$ independent per-tree queries.
We verify Specification~B by checking each of the 100 trees individually. Each query involves a single tree rather than the full ensemble, and every query completes in under 1~second.
All 100 per-tree checks return \texttt{unsat}, proving Specification~B for the full ensemble by \Cref{prop:decomp}.
This per-tree strategy yields a complete formal proof, not an empirical approximation.

%% ------------------------------------------------------------
\subsection{Verify-fix-verify loop}
\label{sec:loop}
%% ------------------------------------------------------------

Verification results guide iterative model improvements:
(1)~train an unconstrained model (82.5\% accuracy);
(2)~verify all specifications (all four violated);
(3)~retrain with \texttt{monotone\_constraints} (GWD$-$, $L-$, PGA$+$);
(4)~re-verify (Specification~B proven via per-tree decomposition, but Specifications~A, C, and D still violated at 69.3\% accuracy).
The loop reveals that monotone constraints are necessary but not sufficient; counterexamples pinpoint where compound violations persist.

%% ------------------------------------------------------------
\subsection{Formal abductive explanations}
\label{sec:abductive}
%% ------------------------------------------------------------

The SMT encoding also enables computing \emph{formal abductive explanations}, the minimal subset of features that, when fixed to a sample's values, guarantees the prediction regardless of how the remaining features vary.

For a sample $\mathbf{x}^*$ predicted as positive ($f(\mathbf{x}^*) > 0$), a subset $S \subseteq \{1, \ldots, d\}$ is a sufficient reason if
\begin{equation}
\forall\, \mathbf{x} \in \mathcal{X}: \Bigl(\bigwedge_{j \in S} x_j = x_j^*\Bigr) \;\Rightarrow\; f(\mathbf{x}) > 0
\end{equation}

We compute this iteratively, starting from the full feature set and testing whether removing each feature preserves the above property via SMT queries.
If removing feature $j$ yields \texttt{unsat} (meaning no counterexample exists even without fixing $j$), the feature is unnecessary and can be dropped.
If \texttt{sat}, some input differing only in $j$ changes the prediction, so $j$ is formally necessary.

Comparing these formal explanations with SHAP's top-$k$ features (where $k$ matches the formal explanation size) measures how faithfully SHAP represents the model's actual decision logic.

%% ============================================================
\section{Results}
\label{sec:results}
%% ============================================================

%% ------------------------------------------------------------
\subsection{Model training}
\label{sec:model-training}
%% ------------------------------------------------------------

We train four model families on the Christchurch dataset (80/20 train-test split, 5{,}832/1{,}459 sites):
(1)~unconstrained XGBoost,
(2)~monotone-constrained XGBoost (GWD$-$, $L-$, PGA$+$),
(3)~unconstrained EBM, and
(4)~monotone-constrained EBM.
For XGBoost, we vary depth $d \in \{3, 5, 7, 9\}$ and estimators $n \in \{20, 50, 100\}$ (24 variants); for EBMs, interaction terms $i \in \{0, 3, 5, 10\}$ (9 variants including the fully constrained model), yielding 33 total.

The baseline unconstrained XGBoost ($d=5, n=100$) achieves 82.5\% test accuracy; 10-fold cross-validation confirms stability ($83.2 \pm 1.4$\% unconstrained, $67.4 \pm 1.0$\% monotone).
The best unconstrained EBM ($i=10$) achieves 80.1\% with a smaller train-test gap, reflecting the GAM structure's inherent regularization.

The SHAP beeswarm plot for XGBoost (\Cref{fig:shap-ebm}, left) reveals that high PGA exhibits both positive and negative SHAP values, the non-physical pattern identified by \citet{hsiao-kumar-2024}.
EBM feature importance (\Cref{fig:shap-ebm}, right) flags a similar concern, but neither visualization can quantify how extensively the model violates physical expectations across its full input domain.

\begin{figure}[htbp]
\centering
\includegraphics[width=\textwidth]{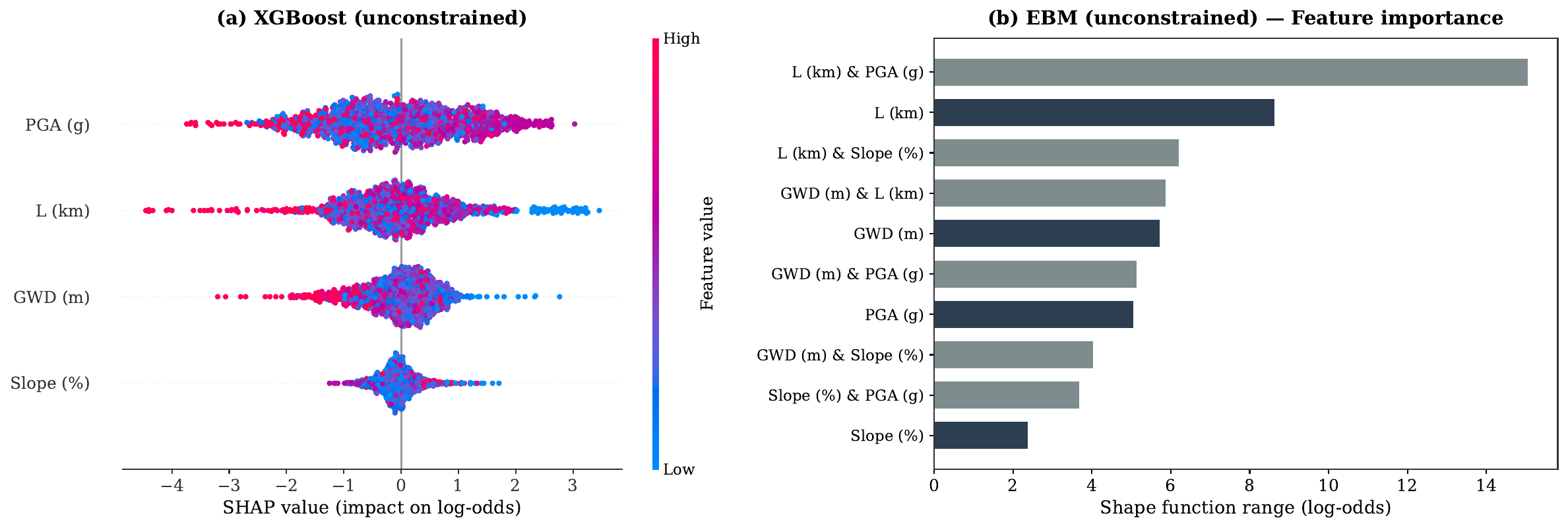}
\caption{Global feature importance. Left: SHAP beeswarm plot for unconstrained XGBoost. Each dot represents a test sample; color indicates feature value (red = high, blue = low). PGA shows both positive and negative SHAP values, indicating non-monotonic behavior. Right: EBM feature importance measured by shape function range (maximum minus minimum log-odds contribution).}
\label{fig:shap-ebm}
\end{figure}

The EBM shape functions (\Cref{fig:ebm-shapes}) reveal the source of these violations.
The unconstrained PGA curve exhibits a dip beyond 0.45~g where increasing PGA decreases predicted log-odds, and the PGA-monotone constraint eliminates this artifact.

\begin{figure}[htbp]
\centering
\includegraphics[width=\textwidth]{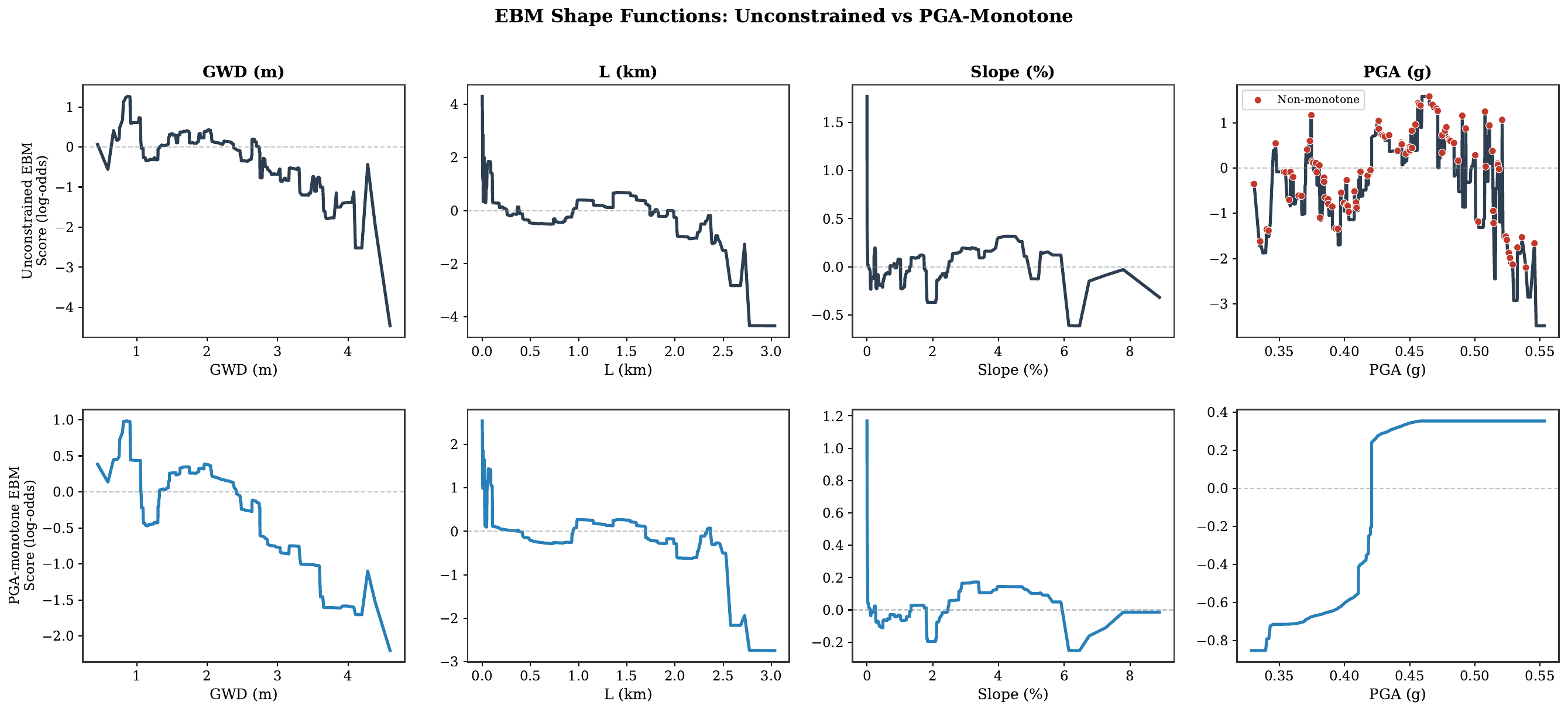}
\caption{EBM shape functions for all four features. Top row: unconstrained EBM. Bottom row: PGA-monotone EBM. The PGA shape function (rightmost column) shows non-physical behavior in the unconstrained model, specifically a dip near 0.42~g where increasing PGA decreases predicted log-odds. The monotone constraint eliminates this artifact.}
\label{fig:ebm-shapes}
\end{figure}

%% ------------------------------------------------------------
\subsection{Specification verification results}
\label{sec:axiom-results}
%% ------------------------------------------------------------

We verified all four specifications against each model family (\Cref{tab:axiom-results}); the following paragraphs discuss each model in detail, with counterexample diagnostics in \Cref{fig:axiom-violations}.

\begin{table}[htbp]
\centering
\caption{Specification verification results. XGBoost results are full SMT proofs (Z3). EBM results use SMT for Specifications~A, C, and D; Specification~B uses per-term score-array decomposition (\Cref{prop:decomp}). The solver checks the \emph{negation} of each specification: if satisfiable, a counterexample exists and the specification is \textbf{violated} (\ding{55}); if unsatisfiable, no counterexample exists and the specification is \textbf{proven} (\ding{51}). All counterexamples are validated against the actual model.}
\label{tab:axiom-results}
\begin{tabular}{lccccl}
\toprule
Model & A (Water table) & B (PGA mono.) & C (Distance) & D (Flat-ground) & Test Acc. \\
\midrule
XGBoost (unconstrained)     & \ding{55} & \ding{55} & \ding{55} & \ding{55} & 82.5\% \\
XGBoost (monotone)          & \ding{55} & \ding{51} & \ding{55} & \ding{55} & 69.3\% \\
EBM (unconstrained)         & \ding{55} & \ding{55} & \ding{55} & \ding{55} & 80.1\% \\
EBM (PGA monotone)          & \ding{55} & \ding{51} & \ding{55} & \ding{55} & 70.5\% \\
EBM (fully constrained)     & \ding{55} & \ding{51} & \ding{51} & \ding{51} & 67.2\% \\
\bottomrule
\end{tabular}
\end{table}

\begin{figure}[htbp]
\centering
\includegraphics[width=\textwidth]{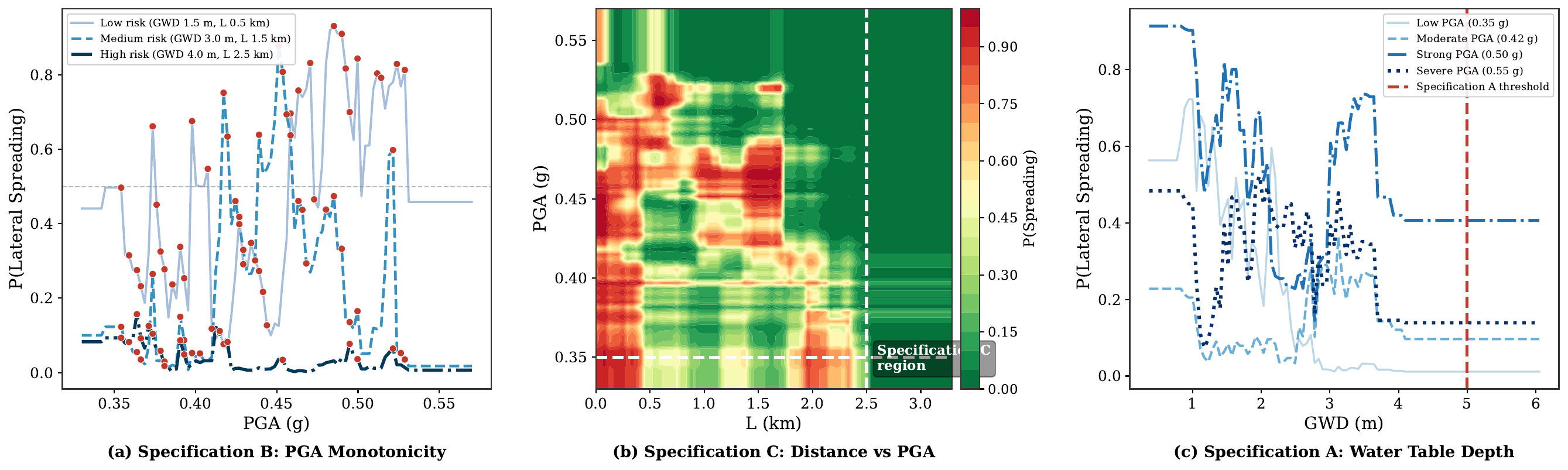}
\caption{Specification violation diagnostics for unconstrained XGBoost. Left: PGA monotonicity (Specification~B) for three soil profiles; red dots mark non-monotone regions. Center: contour of predicted spreading probability over distance ($L$) and PGA, with Specification~C safe region marked. Right: predicted probability versus groundwater depth for several PGA levels, with Specification~A threshold at 5~m.}
\label{fig:axiom-violations}
\end{figure}

\paragraph{Unconstrained XGBoost}
The solver returns \texttt{sat} for all four specifications (each within 1~second).
The Specification~A counterexample (GWD $= 5.001$~m, logit $= +0.065$, 52\%) predicts spreading despite deep groundwater.
The Specification~B counterexample shows increasing PGA from 0.42~g to 0.52~g \emph{decreasing} the logit by 0.15~units, confirming the SHAP-based observation of \citet{hsiao-kumar-2024} by exhaustive formal proof.
The Specification~C counterexample ($L = 2.501$~km, PGA $= 0.33$~g, logit $= +0.169$, 54\%) and Specification~D counterexample (Slope $= 0.05$\%, $L = 2.501$~km, logit $= +0.169$, 54\%) predict spreading in regions where the adopted screening specifications exclude the usual driving conditions and the training data contain no observed spreading cases (\Cref{fig:extrapolation}b).

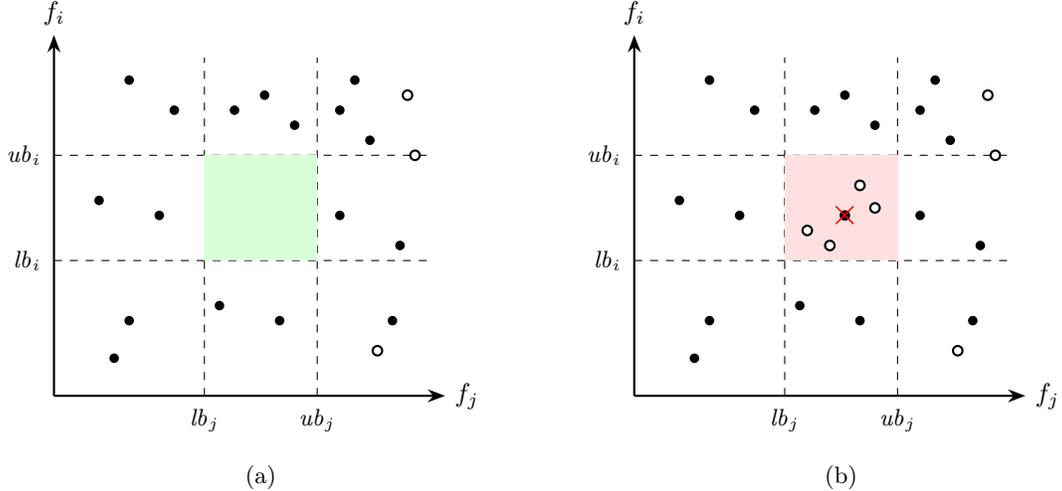
\begin{figure}[htbp]
\centering
\begin{tikzpicture}[>=Stealth]
  % --- Panel (a) ---
  % Axes
  \draw[->, thick] (0,0) -- (5.2,0) node[right] {$f_j$};
  \draw[->, thick] (0,0) -- (0,4.8) node[above] {$f_i$};

  % Dashed threshold lines
  \draw[dashed] (2.0, 0) -- (2.0, 4.5);
  \draw[dashed] (3.5, 0) -- (3.5, 4.5);
  \draw[dashed] (0, 1.8) -- (5.0, 1.8);
  \draw[dashed] (0, 3.2) -- (5.0, 3.2);

  % Threshold labels
  \node[below, font=\small] at (2.0, -0.05) {$\mathit{lb}_j$};
  \node[below, font=\small] at (3.5, -0.05) {$\mathit{ub}_j$};
  \node[left, font=\small] at (-0.05, 1.8) {$\mathit{lb}_i$};
  \node[left, font=\small] at (-0.05, 3.2) {$\mathit{ub}_i$};

  % Green shaded box (data gap)
  \fill[green!15] (2.0, 1.8) rectangle (3.5, 3.2);

  % Data points (all outside the green box)
  % Filled = positive class
  \foreach \x/\y in {
    1.0/4.2, 1.6/3.8, 2.4/3.8, 2.8/4.0, 3.2/3.6,
    0.6/2.6, 1.4/2.4, 2.2/1.2, 3.0/1.0, 1.0/1.0,
    3.8/2.4, 4.2/3.4, 4.6/2.0, 4.0/4.2, 4.5/1.0,
    3.8/3.8, 0.8/0.5} {
    \fill[black] (\x, \y) circle (1.8pt);
  }
  % Open = negative class
  \foreach \x/\y in {4.8/3.2, 4.3/0.6, 4.7/4.0} {
    \draw[black, fill=white, thick] (\x, \y) circle (1.8pt);
  }

  \node[font=\small, anchor=north] at (2.75, -0.8) {(a)};
\end{tikzpicture}%
\hspace{12mm}%
\begin{tikzpicture}[>=Stealth]
  % --- Panel (b) ---
  % Axes
  \draw[->, thick] (0,0) -- (5.2,0) node[right] {$f_j$};
  \draw[->, thick] (0,0) -- (0,4.8) node[above] {$f_i$};

  % Dashed threshold lines
  \draw[dashed] (2.0, 0) -- (2.0, 4.5);
  \draw[dashed] (3.5, 0) -- (3.5, 4.5);
  \draw[dashed] (0, 1.8) -- (5.0, 1.8);
  \draw[dashed] (0, 3.2) -- (5.0, 3.2);

  % Threshold labels
  \node[below, font=\small] at (2.0, -0.05) {$\mathit{lb}_j$};
  \node[below, font=\small] at (3.5, -0.05) {$\mathit{ub}_j$};
  \node[left, font=\small] at (-0.05, 1.8) {$\mathit{lb}_i$};
  \node[left, font=\small] at (-0.05, 3.2) {$\mathit{ub}_i$};

  % Red shaded box (contradicted evidence)
  \fill[red!12] (2.0, 1.8) rectangle (3.5, 3.2);

  % Data points outside the red box
  % Filled = positive class
  \foreach \x/\y in {
    1.0/4.2, 1.6/3.8, 2.4/3.8, 2.8/4.0, 3.2/3.6,
    0.6/2.6, 1.4/2.4, 2.2/1.2, 3.0/1.0, 1.0/1.0,
    3.8/2.4, 4.2/3.4, 4.6/2.0, 4.0/4.2, 4.5/1.0,
    3.8/3.8, 0.8/0.5} {
    \fill[black] (\x, \y) circle (1.8pt);
  }
  % Open = negative class outside
  \foreach \x/\y in {4.8/3.2, 4.3/0.6, 4.7/4.0} {
    \draw[black, fill=white, thick] (\x, \y) circle (1.8pt);
  }

  % Open circles inside the red box (negative class only)
  \foreach \x/\y in {2.3/2.2, 3.0/2.8, 2.6/2.0, 3.2/2.5} {
    \draw[black, fill=white, thick] (\x, \y) circle (1.8pt);
  }

  % Counterexample (tick / cross mark)
  \fill[black] (2.8, 2.4) circle (2pt);
  \node[red, font=\Large] at (2.8, 2.4) {$\times$};

  \node[font=\small, anchor=north] at (2.75, -0.8) {(b)};
\end{tikzpicture}
\caption{Schematic of two failure modes in a two-dimensional feature subspace defined by bounds $\mathit{lb}$ and $\mathit{ub}$.
(a)~No training data fall in the green region; the model interpolates without empirical support.
(b)~Training observations exist in the red region but are exclusively negative-class ($\circ$), yet the model predicts $y = 1$ at the SMT counterexample ($\times$).
Specification~D violations correspond to case~(b).}
\label{fig:extrapolation}
\end{figure}

\paragraph{Monotone-constrained XGBoost}
Monotone constraints drop accuracy to 69.3\% ($-$13.2~pp).
The Specification~A counterexample (GWD $= 5.014$~m, $L = 0.011$~km, PGA $= 0.33$~g, logit $= +0.221$) shows that per-feature constraints cannot encode the compound reasoning that deep groundwater should override proximity and shaking.
Specification~B is proven in two independent ways, per-tree decomposition (100 trees, 0.6~s total) and direct dual-ensemble verification (\texttt{unsat} in 480~s), but Specifications~C and D still return \texttt{sat}, confirming that monotone constraints fix single-feature violations while leaving compound specifications unsatisfied.

\paragraph{EBM variants}
The unconstrained EBM violates all four specifications.
Specification~A is violated (GWD $= 5.000$~m, $L = 0.005$~km, logit $= +1.68$, 84\%), despite the EBM's additive structure, because the combined positive contributions from proximity ($L$), slope, and PGA outweigh the GWD shape function's negative contribution at deep water table depths.
Specification~B is violated (272 non-monotone segments in the PGA shape function).
Specification~C is violated ($L = 2.52$~km, PGA $= 0.33$~g, logit $= +0.13$, 53\%); this query requires the specialized QF\_LRA solver and completes in 1{,}331~s.
Specification~D is violated (Slope $= 0.001$\%, $L = 2.72$~km, logit $= +0.03$, 51\%).

The PGA-monotone EBM proves Specification~B via per-term decomposition but still violates Specification~A (GWD $= 5.00$~m, $L = 0.005$~km, logit $= +0.14$, 53\%), with accuracy dropping 9.6~pp (80.1\% $\rightarrow$ 70.5\%).

\paragraph{Fully constrained EBM}
To test whether more specifications can be satisfied simultaneously, we train an EBM with monotone constraints on all four features (GWD decreasing, $L$ decreasing, Slope increasing, PGA increasing), enforcing the physical directions for every feature.
This model proves three of four specifications, namely Specification~B (PGA monotonicity, via per-term decomposition), Specification~C (distance safety), and Specification~D (flat-ground safety), the latter two proven by the SMT solver returning \texttt{unsat}.
However, Specification~A (water table depth) is violated by a thin margin; the solver finds a counterexample at GWD $= 5.003$~m, $L = 0.003$~km, Slope $= 2.5$\%, PGA $= 0.46$~g, with a predicted probability of 50.3\% (logit $= +0.012$).
Even with monotonicity constraints on all four features, the model cannot guarantee Specification~A because monotonicity enforces the \emph{direction} of a feature's effect, not an absolute threshold on the output.
The accuracy cost of full constraint is 12.9 percentage points (80.1\% $\rightarrow$ 67.2\%), but the model achieves the highest specification compliance of any variant tested.
\Cref{fig:pga-mono} illustrates the effect. For a fixed soil profile, unconstrained XGBoost and EBM exhibit clear non-monotone PGA regions, while the constrained variants produce strictly increasing curves.

\begin{figure}[htbp]
\centering
\includegraphics[width=\textwidth]{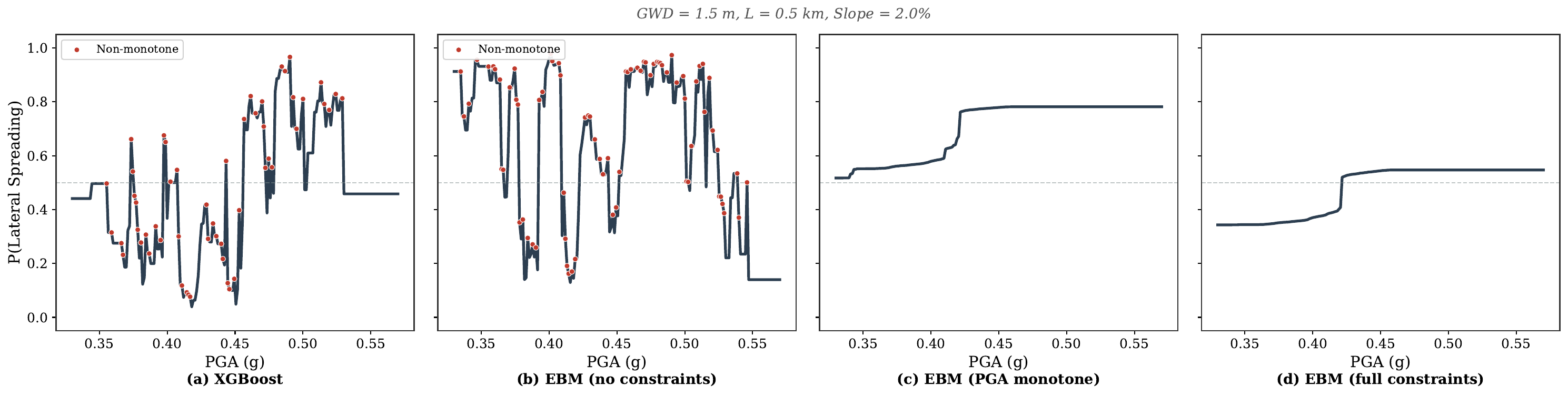}
\caption{PGA monotonicity comparison across model families for a fixed soil profile (GWD = 1.5~m, $L$ = 0.5~km, Slope = 2.0\%). Red dots mark non-monotone regions where increasing PGA decreases predicted probability. Unconstrained XGBoost and EBM exhibit violations; monotone-constrained variants do not.}
\label{fig:pga-mono}
\end{figure}

%% ------------------------------------------------------------
\subsection{Pareto frontier: accuracy versus physical consistency}
\label{sec:pareto}
%% ------------------------------------------------------------

To quantify the accuracy cost of physical consistency, we trained all 33 model variants and measured Specification~B compliance as the fraction of 120 empirically sampled PGA-sweep profiles satisfying monotonicity (\Cref{fig:pareto}).
This empirical proxy enables broad hyperparameter comparison; formal SMT proofs (\Cref{tab:axiom-results}) are applied to representative models in each family.

\begin{figure}[htbp]
\centering
\includegraphics[width=0.8\textwidth]{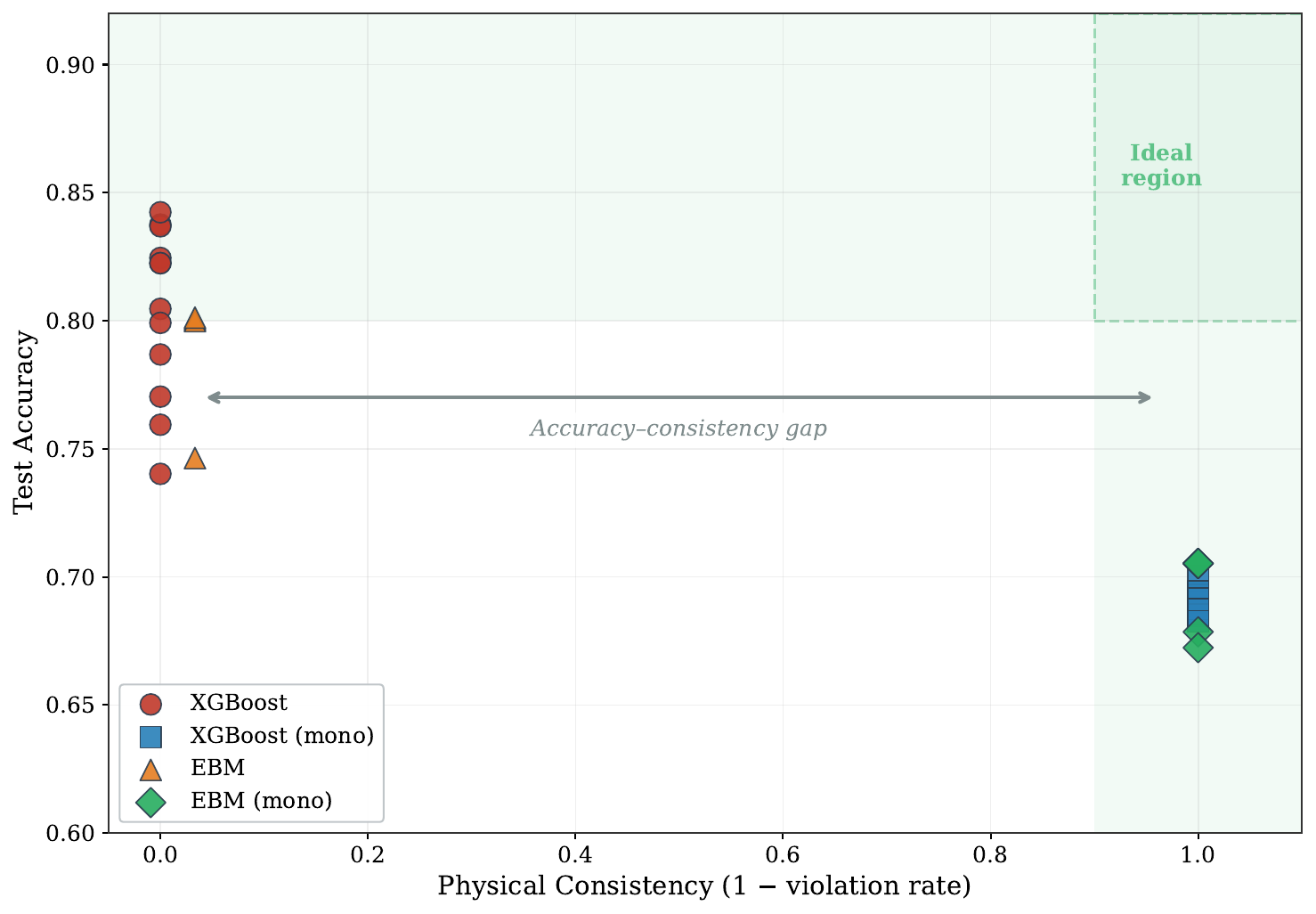}
\caption{Accuracy versus physical consistency (Specification~B compliance) for 33 model variants. Unconstrained models cluster at high accuracy with zero compliance (upper-left); constrained models cluster at lower accuracy with full compliance (lower-right). The fully constrained EBM (lowest green diamond, 67.2\%) achieves full Specification~B compliance and proves 3 of 4 specifications via SMT. No model occupies the ideal top-right region.}
\label{fig:pareto}
\end{figure}

Unconstrained models cluster at high accuracy (74--84\%) with near-zero compliance; constrained models cluster at lower accuracy (68--71\%) with full compliance.
No model occupies the top-right region where accuracy exceeds 80\% and compliance reaches 100\%, consistent with unconstrained models fitting patterns that violate the specifications.
\Cref{tab:pareto-summary} confirms this pattern across all 33 variants. The 10--15 percentage point accuracy gap persists regardless of model family or hyperparameter setting.

\begin{table}[htbp]
\centering
\caption{Summary statistics for the Pareto frontier, aggregated across hyperparameter variants. Violations and compliance refer to Specification~B (PGA monotonicity) only, measured on 120 empirically sampled PGA-sweep profiles. Full specification verification results (Specifications~A--D) are in \Cref{tab:axiom-results}.}
\label{tab:pareto-summary}
\begin{tabular}{lcccc}
\toprule
Model family & Variants & Acc. range & B violations & B compliance \\
\midrule
XGBoost (unconstrained) & 12 & 74.0--84.2\% & 120/120 & 0\% \\
XGBoost (monotone)      & 12 & 68.3--70.0\% & 0/120  $\checkmark$ & 100\% $\checkmark$\\
EBM (unconstrained)     & 4  & 74.6--80.1\% & 116/120 & 3.3\% \\
EBM (monotone)          & 4  & 67.9--70.5\% & 0/120 $\checkmark$   & 100\% $\checkmark$\\
EBM (fully constrained) & 1  & 67.2\%       & 0/120 $\checkmark$   & 100\% $\checkmark$\\
\bottomrule
\end{tabular}
\end{table}

%% ------------------------------------------------------------
\subsection{SHAP versus formal explanations}
\label{sec:shap-formal}
%% ------------------------------------------------------------

SHAP explains \emph{how} a model arrives at a prediction by decomposing it into per-feature contributions.
Formal verification answers a different question, namely \emph{does} the prediction comply with physical specifications?
We show that these two questions can give contradictory signals.

\Cref{fig:shap-formal-cases} illustrates this gap using a Specification~A counterexample.
The input has GWD $= 5.65$~m, $L = 1.21$~km, Slope $= 4.86$\%, PGA $= 0.51$~g.
The model predicts spreading with 91\% confidence (logit $= +2.31$), violating Specification~A (GWD $> 5$~m should not predict spreading).
Panel~(a) shows the standard SHAP waterfall that an engineer would inspect.
$L$ contributes $+1.67$, PGA $+0.61$, Slope $+0.52$, and GWD only $-0.12$.
An engineer would read this as ``close to a free face, strong shaking, steep slope'' and find the prediction reasonable.
GWD, the feature that makes this prediction a specification violation, is ranked last by SHAP with the smallest magnitude.

\begin{figure}[t]
\centering
\includegraphics[width=\textwidth]{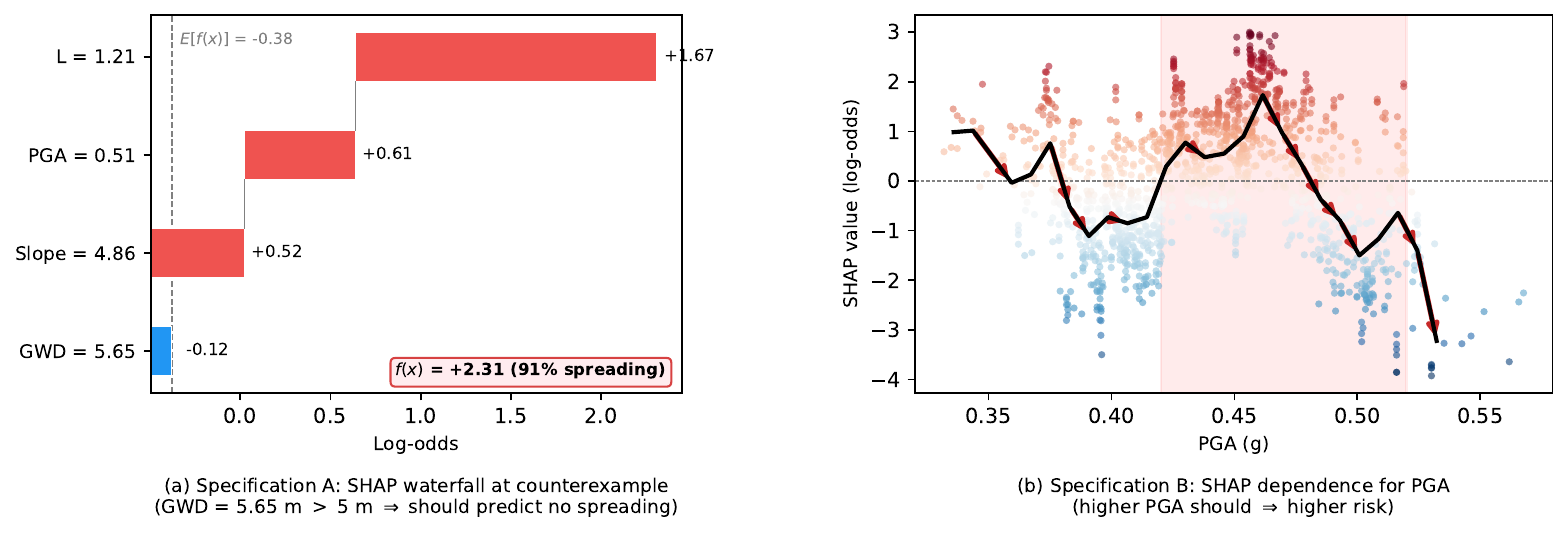}
\caption{SHAP analysis of specification violations (unconstrained XGBoost). (a)~Waterfall at a Specification~A counterexample (GWD $= 5.65$~m, predicted 91\% spreading). Red bars push toward spreading, blue bars push against it. GWD is the feature that violates Specification~A, yet SHAP ranks it last (\#4) with the smallest magnitude ($-0.12$). (b)~SHAP dependence plot for PGA. Each dot is one test sample; color indicates the SHAP value (red = pushes toward spreading, blue = pushes against). The black line (binned mean) should increase monotonically if the model respects Specification~B; instead, it drops sharply between 0.42--0.52~g (shaded region, red arrows), confirming the non-monotonicity that formal verification detects.}
\label{fig:shap-formal-cases}
\end{figure}

The underlying issue is that SHAP decomposes predictions, not validates them.
SHAP correctly reports that GWD contributes $-0.12$ log-odds at this input, but it ranks  the GWD feature last, indicating its minimal impact. 
In contrast, formal verification proves that any input with GWD $> 5$~m violates Specification~A, regardless of the contributions from other features.
While SHAP explains the model's reasoning for a specific input, it cannot detect that the model's overall logic violations against a specification.
The same pattern appears for Specification~B (\Cref{fig:shap-formal-cases}b).
The SHAP dependence plot for PGA shows that the model's SHAP values drop sharply between 0.42--0.52~g instead of increasing monotonically.
While the global dependence plot makes this visible, individual waterfall plots would not detect the violation, because each waterfall appears locally reasonable.
Formal verification detects the violation automatically and returns a concrete counterexample.

%% ------------------------------------------------------------
\subsection{Counterexample analysis}
\label{sec:counterexamples}
%% ------------------------------------------------------------

Each SMT counterexample is a concrete input vector that the solver constructs to violate the specification.
Validating each counterexample against the actual model confirms the encoding's fidelity; \Cref{tab:counterexamples} lists representative counterexamples for the unconstrained XGBoost model.

\begin{table}[htbp]
\centering
\caption{Representative counterexamples from SMT verification of unconstrained XGBoost. A positive logit (probability $> 50$\%) means the model predicts spreading, violating the specification. All counterexamples were validated against the actual XGBoost model to machine precision.}
\label{tab:counterexamples}
\begin{tabular}{lcccccc}
\toprule
Spec. & GWD (m) & $L$ (km) & Slope (\%) & PGA (g) & Logit & Probability \\
\midrule
A (water table)      & 5.001 & 1.35 & 1.58 & 0.37 & +0.065 & 52\% \\
C (distance safety)  & 1.73 & 2.501 & 0.05 & 0.33 & +0.169 & 54\% \\
D (flat-ground)      & 1.72 & 2.501 & 0.05 & 0.33 & +0.169 & 54\% \\
\bottomrule
\end{tabular}
\end{table}

The Specification~A counterexample (GWD $= 5.001$~m) predicts spreading at 52\% probability because the model relies on proximity ($L = 1.35$~km) while insufficiently weighting groundwater depth.
The Specification~C counterexample ($L = 2.501$~km, PGA $= 0.33$~g) predicts spreading at 54\% in a region where the Christchurch data contain no observed spreading cases, because the ensemble accumulates positive contributions where training data is sparse.
Both counterexamples occur in data-sparse boundary regions, connecting to the root cause identified by \citet{hsiao-kumar-rathje-2025}.
Formal verification transforms the intuition that ``sparse data causes problems'' into a concrete, validated demonstration.

%% ============================================================
\section{Discussion and conclusions}
\label{sec:discussion}
%% ============================================================

\paragraph{Formal verification versus empirical testing}
Grid-based testing can detect violations but cannot certify their absence.
\Cref{tab:grid-vs-smt} shows that uniform grid search with up to $n = 200$ points per feature (1.6~billion evaluations) finds \emph{zero} Specification~C violations, while the SMT solver finds a counterexample in 0.19~s.
The violation region occupies a thin sliver ($L \in [2.500, 2.505]$~km, Slope~$< 0.05$\%) smaller than a single grid step.
When the solver returns \texttt{unsat}, the specification holds universally, a guarantee qualitatively different from any finite-sample test.

\begin{table}[htbp]
\centering
\caption{Grid search versus SMT verification for Specifications~C and D (unconstrained XGBoost). Grid search evaluates the model on $n^4$ uniformly spaced points over $\mathcal{X}$. The SMT solver encodes the model symbolically and reasons over the entire continuous domain. Specification~D adds a Slope $< 0.1$\% constraint, which further reduces the number of grid points in the premise region.}
\label{tab:grid-vs-smt}
\begin{tabular}{lrrrrr}
\toprule
Method & Points eval. & Premise (C) & Premise (D) & Time & Violated? \\
\midrule
Grid ($n = 30$)  & 810{,}000         & 18{,}900         & 630          & $<$0.1\,s & None found \\
Grid ($n = 50$)  & 6{,}250{,}000     & 150{,}000        & 3{,}000      & 0.1\,s    & None found \\
Grid ($n = 100$) & 100{,}000{,}000   & 2{,}160{,}000    & 21{,}600     & 1.7\,s    & None found \\
Grid ($n = 150$) & 506{,}250{,}000   & 10{,}530{,}000   & 140{,}400    & 17\,s     & None found \\
Grid ($n = 200$) & 1{,}600{,}000{,}000 & 32{,}640{,}000 & 326{,}400  & 184\,s    & None found \\
\midrule
SMT (C)          & \multicolumn{3}{c}{\emph{entire continuous domain}} & 0.19\,s & \textbf{Yes} \\
SMT (D)       & \multicolumn{3}{c}{\emph{entire continuous domain}}  &  0.7\,s & \textbf{Yes} \\
\bottomrule
\end{tabular}
\end{table}

\paragraph{Scalability}
XGBoost verification completes in under 1~s per specification for single-ensemble queries.
Per-tree decomposition reduces monotonicity proofs by orders of magnitude compared to the direct dual-ensemble encoding.
EBM verification is slower due to larger encodings; the most expensive query requires a specialized QF\_LRA solver.
Higher-dimensional models will require abstraction techniques \citep{varshney-2026,kantchelian-2016} to maintain practical runtimes.

\paragraph{Compound constraints and constraint interactions}
Monotone constraints enforce the direction of a feature's marginal effect but cannot guarantee threshold behavior (Specification~A) or multi-feature conjunctions (Specification~C).
Progressively adding constraints improves compliance, but some specifications resist full resolution because monotonicity enforces the \emph{direction} of a feature's effect, not an absolute threshold on the output.
Promising extensions include counterexample-guided data augmentation and post-hoc output clamping; this paper establishes the diagnostic prerequisite by identifying \emph{what} fails and \emph{where}.

\paragraph{Why not embed hard constraints and skip verification}
Available constraint mechanisms operate per-feature and cannot encode compound conditions (Specification~C) or threshold guarantees (Specification~A).
Our results confirm that even the most constrained model cannot satisfy all specifications.
SMT counterexamples pinpoint \emph{where} a model fails, providing actionable information that a pass/fail constraint flag cannot.
Formal verification therefore complements hard constraints.  Constraints restrict the hypothesis space, while verification determines whether the resulting trained model satisfies the specification.

\paragraph{Deterministic versus probabilistic guarantees}
The SMT solver proves that a specification holds for \emph{all} inputs or provides a single counterexample where it fails.
This deterministic guarantee applies to the \emph{model}, not to the physical system; it says nothing about whether $f$ correctly predicts the real-world outcome, which depends on data quality, feature sufficiency, and epistemic uncertainty.
A natural extension would combine formal verification with probabilistic calibration (e.g., conformal prediction), providing complementary guarantees of physical consistency and well-calibrated uncertainty.
Data-aware verification \citep{varshney-2026}, which constrains counterexamples to remain near the training distribution, offers a middle ground.

\paragraph{Specification selection and scope}
The four specifications exercise structurally different formula classes, including a single-feature threshold (Specification~A), monotonicity (Specification~B), and compound multi-feature screening conditions (Specifications~C and D).
This distinction matters because standard monotone constraints can directly enforce only Specification~B.
A production specification library would be richer, incorporating CPT-based susceptibility, depth-weighted liquefaction potential, or magnitude-duration effects, but each new specification becomes an independent SMT query with no framework changes.
This study is methodological. It demonstrates that SMT-based formal verification can check tree ensembles against geotechnical specifications, rather than establishing a new predictor or regional model.

\begin{table}[t]
\centering
\caption{Threshold sensitivity for Specification~A (water table depth) against unconstrained XGBoost. The model violates Specification~A at every GWD threshold.}
\label{tab:threshold-sensitivity}
\begin{tabular}{ccc}
\toprule
GWD threshold (m) & Result & Counterexample logit \\
\midrule
3.0 & Violated & $+0.131$ \\
3.5 & Violated & $+1.405$ \\
4.0 & Violated & $+0.062$ \\
4.5 & Violated & $+0.366$ \\
5.0 & Violated & $+1.405$ \\
5.5 & Violated & $+1.405$ \\
6.0 & Violated & $+0.045$ \\
\bottomrule
\end{tabular}
\end{table}

\paragraph{Limitations and future directions}
The specifications encode necessary conditions, not sufficient ones.
A compliant model could still learn non-physical patterns outside the specification set.
The approach applies only to tree-based models; neural network verification via MILP encodings \citep{kantchelian-2016} is possible but computationally more expensive.
The thresholds in Specifications~A and C are site-specific Christchurch bounds, not universal constants; applying the framework elsewhere requires recalibrating thresholds to local conditions.
\Cref{tab:threshold-sensitivity} demonstrates that the unconstrained model violates Specification~A across a range of thresholds, confirming a persistent pattern rather than a narrow edge case.
Extending the framework to regression models predicting displacement magnitudes is conceptually straightforward in SMT but requires reformulating specifications over continuous outputs.

In summary, SMT-based formal verification can detect physical inconsistencies in tree ensemble models that empirical testing and post-hoc explanations miss.
Training-time constraints fix targeted violations but cannot enforce compound or threshold specifications, and physical consistency comes at a quantifiable accuracy cost.
The verify-fix-verify loop (train for accuracy, verify against specifications, apply targeted constraints, re-verify) provides a systematic workflow for deploying physically consistent ML models in safety-critical geotechnical applications.
As ML models increasingly inform engineering decisions, formal verification offers a path toward certifying that these models respect established specifications before field deployment.

%% ============================================================
\section*{Data availability}
%% ============================================================

The Christchurch lateral spreading dataset is available from \citet{durante-rathje-2021}.
The verification code and trained models are available at \url{https://github.com/geoelements-dev/2026-formal-verify-liq}.

%% ============================================================
\bibliography{references}
\bibliographystyle{elsarticle-harv}

\end{document}